%% file: main.tex
\documentclass{ecai} 

\usepackage{amsmath}
\usepackage{amssymb}
\usepackage{mathtools}
\usepackage{amsthm}
\usepackage{graphicx}
\usepackage{subfigure}
\usepackage{algorithm}
\usepackage[noend]{algpseudocode}
\usepackage{booktabs}
\usepackage[most]{tcolorbox}
\usepackage{bbm}

\renewcommand{\vec}{\boldsymbol}
\newcommand{\allzeros}{\vec 0}
\newcommand{\mat}{\boldsymbol}
\newcommand{\round}[1]{\lfloor #1 \rceil}

\DeclareMathOperator*{\argmax}{arg\,max}
\usepackage{amsfonts}

\newtheorem{theorem}{Theorem}[section]

\newtheorem{proposition}[theorem]{Proposition}
\newtheorem{remark}[theorem]{Remark}

\begin{document}

%%%%%%%%%%%%%%%%%%%%%%%%%%%%%%%%%%%%%%%%%%%%%%%%%%%%%%%%%%%%%%%%%%%%%%%%

\begin{frontmatter}

\title{The Subset Sum Matching Problem}

\author[A]{\fnms{Yufei}~\snm{Wu}\footnote{Equal contribution.}}
\author[B]{\fnms{Manuel R.}~\snm{Torres}\thanks{Corresponding Author. Email: manuel.torres@jpmchase.com.}\footnotemark}
\author[B]{\fnms{Parisa}~\snm{Zehtabi}}
\author[B]{\fnms{Alberto}~\snm{Pozanco Lancho}}
\author[B]{\fnms{Michael}~\snm{Cashmore}}
\author[B]{\fnms{Daniel}~\snm{Borrajo}}
\author[B]{\fnms{Manuela}~\snm{Veloso}}

\address[A]{J.P. Morgan Quantitative Research}
\address[B]{J.P. Morgan AI Research}

\begin{abstract}
This paper presents a new combinatorial optimisation task, the Subset Sum Matching Problem (SSMP), which is an abstraction of common financial applications such as trades reconciliation. We present three algorithms, two suboptimal and one optimal, to solve this problem. We also generate a benchmark to cover different instances of SSMP varying in complexity, and carry out an experimental evaluation to assess the performance of the approaches.
\end{abstract}

\end{frontmatter}

\input{1_introduction}

\input{2_3_problem_description}
\input{4_related_works}
\input{5_6_solution}
\input{7_experiment_result}

\input{8_discussion}

\section{Disclaimer}
This paper was prepared for informational purposes in part by
the Artificial Intelligence Research group of JPMorgan Chase \& Co. and its affiliates (``JP Morgan''),
and is not a product of the Research Department of JP Morgan.
JP Morgan makes no representation and warranty whatsoever and disclaims all liability,
for the completeness, accuracy or reliability of the information contained herein.
This document is not intended as investment research or investment advice, or a recommendation,
offer or solicitation for the purchase or sale of any security, financial instrument, financial product or service,
or to be used in any way for evaluating the merits of participating in any transaction,
and shall not constitute a solicitation under any jurisdiction or to any person,
if such solicitation under such jurisdiction or to such person would be unlawful.

\section*{Acknowledgments}
We thank the anonymous reviewers for their constructive feedback, which improved the presentation of our results and in particular Theorem~\ref{thm:dp}.

\bibliography{refs}

\newpage
\appendix
\onecolumn

\input{appendix}

\end{document}

%% file: 1_introduction.tex
\section{Introduction}

Combinatorial Optimisation (CO) problems aim to find an optimal configuration over a discrete domain of possibilities. This paper proposes and formulates a new CO problem, the Subset Matching Problem (SMP). At a high level, the input to SMP is two multisets of objects, a Boolean function that determines if exactly one subset from each multiset form a so-called \emph{match}, and an objective function defined over the feasible solutions. Feasible solutions correspond to sets of disjoint pairs of matches and the goal is to maximize the given objective function over such solutions. 

There are many CO problems with similarities to SMP but no direct equivalent. In the generalised assignment problem~\cite{oncan2007survey}, for example, the goal is to take two sets and create matches with exactly one element from both sets subject to packing constraints. This has similarities to SMP but is different in many respects, including limiting the size of the matches made. Many CO problems also include in their definition a given function that checks whether a match is valid. For instance, in the subset sum problem~\cite{kleinberg2006algorithm}, a match is valid when the sum of the values of the elements on one side is equal to a given target value. SMP is also similar to the hypergraph matching problem~\cite{schrijver2003combinatorial}, where if we partition the vertices into two sets, we can view a hyperedge as a match. However, matches in SMP are formed from subsets of the input multisets, so these subsets are not direct inputs to the problem, a difference from the hypergraph matching problem.
SMP outputs a set of matches, without any constraint on the size of the subsets in each match. SMP also is independent of the function used to define the validity of a given match as the function is an input to the problem. Finally, as opposed to other problems computing several matches, SMP allows some elements to not be part of a match in a feasible solution.

Based on the formulation of SMP, we introduce the Subset Sum Matching Problem (SSMP), an instantiation of SMP where the Boolean validation function focuses on subset sums (see Fig.~\ref{fig:ssmp} for an example instance and three possible feasible solutions). Note that to formally specify instances of SMP, we have to (a) define the types of objects in the multisets and (b) define the Boolean function used to determine matches. For SSMP, the types of objects in the multisets are real values, and we say two subsets form a match if the absolute difference of the sums is within some given error tolerance $\epsilon$.

\begin{figure}[t]
\centering
\includegraphics[width=0.45\textwidth]{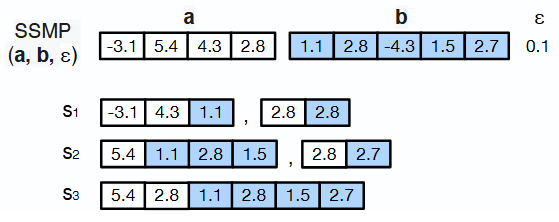}
\caption{\label{fig:ssmp} An example of SSMP with input multisets  $\vec a, \vec b$ and tolerance $\epsilon$. $s_1, s_2, s_3$ indicate three different feasible solutions. For example, $s_2$ is a feasible solution consisting of two matches. The first match contains one element from $\vec a$, three elements from vector $\vec b$, and is indeed a match because $|5.4 - (1.1+2.8+1.5)| \le \epsilon = 0.1$}
\end{figure}

While the general SMP has many applications in the context of task assignment, for example, matching a set of tasks to a set of workers given some skill constraints, this paper focuses on SSMP.
SSMP has many real-world applications. In particular, it is a critical part of an accounting process known as reconciliation, where two sets of financial records are compared to ensure numerical accuracy and agreement. Minor discrepancies may be accepted and explained due to timing differences in the processing of payments and deposits, or as a result of data entry errors.
To connect reconciliation back to SSMP, the two multisets of objects are the two sets of financial records and the matching function that checks the absolute difference of sums allows for such minor discrepancies.
Reconciliation tasks, for example, inter-company or customer reconciliations, are labour intensive, and are essential for helping businesses and individuals to confirm that accounts are consistent and complete. They are also fundamental to detect potential fraudulent activities or errors via discovering discrepancies between two financial records or account balances. Reconciliation is recognized as an essential internal control mechanism, ensuring data integrity and consistency across financial record systems~\cite{ge2005disclosure,vorhies2006account}.

The main contributions of this paper are:
\begin{itemize}
    \item We propose new CO problems, the Subset Matching Problem (SMP) and the Subset Sum Matching Problem (SSMP), and provide problem formulations for them. 
    \item We introduce three algorithms, some optimal and some sub-optimal, to solve SSMP tasks. 
    \item We define a benchmark composed of several SSMP tasks and run experiments to show how different types of algorithms perform on the benchmark. 
\end{itemize}

%% file: 2_3_problem_description.tex
\section{Subset Matching Problem}

We start by formally defining the Subset Matching Problem (SMP). 

\subsection{Problem Definition}
Assume $\vec a \in \mathcal{A}^M$, $\vec b \in \mathcal{B}^N$ are two ordered lists of items.\footnote{We do not specify $\mathcal{A}, \mathcal{B}$ here to ensure the problem description remains general. We only require that the elements in $\mathcal{A}$ and $\mathcal{B}$ support addition and multiplication.} Subsets of $\vec a$ and $\vec b$ are represented with $\vec a_{\vec w}$ and $\vec b_{\vec v}$, where $\vec w \in \{0,1\}^M$, $\vec v \in \{0,1\}^N$ are \emph{inclusion vectors} indicating inclusion/exclusion of each element in that subset corresponding to ordered lists $\vec a$, $\vec b$ (technically, these are multisets as we allow repetition in $\vec a$ and $\vec b$). Let $f$ be a Boolean function that 
determines whether a pair of subsets forms a valid match. Hence, a (valid) \emph{match} is a pair of subsets $\langle \vec a_{\vec w}, \vec b_{\vec v} \rangle$ that satisfies $f(\vec a_{\vec w}, \vec b_{\vec v}) = $ True.

The Subset Matching Problem (SMP) aims to find a set of \emph{matches} with non-overlapping elements that maximize some metric. Let $\mathcal{S}$ be the finite set of all feasible solutions to an SMP instance -- each feasible solution $s \in \mathcal{S}$ can be expressed as a combination of $K > 0$ matches 

\begin{equation*}
s = \{\langle\vec a_{\vec w^k}, \vec b_{\vec v^k}\rangle : k=1,...,K\}
\end{equation*}
where inclusion vectors in all \emph{matches} do not overlap:
\begin{equation}
\label{eq:SMP_matching_criteria}
    \forall i,j \in [1, K], i \ne j, \vec w^i \cdot \vec w^j = 0, \vec v^i \cdot \vec v^j = 0
\end{equation}
where $\cdot$ is the dot product.
We define SMP as a combinatorial optimisation problem $\text{SMP}(\vec a, \vec b, f, \Psi)$,
where $\Psi: \mathcal{S} \rightarrow \mathbb{R}$ is a \emph{measure} of the solution quality to be maximised. An optimal solution is:
\begin{equation}
\label{eq:optimal}
    s^* = \argmax_{s \in \mathcal{S}} \Psi(s) 
\end{equation}

\subsection{Corresponding Decision Problem}

We define $\text{SMP}^-(\vec a, \vec b, f)$
as the decision problem corresponding to the SMP CO problem. $\text{SMP}^-$ can be answered with (a) `yes', by finding (any) \emph{one} match or (b) `no', indicating that no feasible solution exists (i.e., $|\mathcal{S}| = 0$). As a matter of notation, let $\perp$ indicate that no feasible solution exists. The set of feasible solutions of the corresponding $\text{SMP}^-$ instance, denoted as $\mathcal{S}^-$, is a subset of the corresponding $\mathcal{S}$:
\begin{equation}
    \mathcal{S}^- = \{s^- \in \mathcal{S}: |s^-| \le 1 \}.
\end{equation}

\section{The Subset Sum Matching Problem}
The Subset Sum Matching Problem (SSMP) is an instantiation of SMP, which formalizes a group of real-world problems like reconciliation in the financial domain. We present in this paper a formulation of SSMP that derives from our definition of SMP. In particular, the matching function $f$ is defined as checking whether the differences between the sums of the numbers of each match on the two sides is below a given threshold, given by a parameter $\epsilon$:

\begin{align}
\label{eq:SSMP_matching_criteria}
    f_{\text{SSMP}}(\vec a_{\vec w}, \vec b_{\vec v}, \epsilon)\!=\!(|\vec w\!\cdot\!\vec a - \vec v\!\cdot\!\vec b| \leq \epsilon) \wedge (\vec w, \vec v\not= \allzeros)
\end{align}
where $\allzeros$ represents the all-zeros vector.
In this paper, $\Psi$ is defined as:
\begin{align}
\label{eq:SSMP_target_function}
        \Psi_{\text{SSMP}}(s) =  \sum_{k=1}^{K} \left[\sum_{m=1}^M\vec w^k_m +\sum_{n=1}^N\vec v^k_n\right] + K
\end{align}
to encourage solutions to cover more elements with finer-grained matches. One could naturally consider scaling $K$ in $\Psi_{\text{SSMP}}$ by a tunable parameter. However, this, and many other alternatives one might consider, is left for future work. Considering Eq. \eqref{eq:SSMP_matching_criteria} and \eqref{eq:SSMP_target_function}, we define:
\begin{equation*}
    \text{SSMP}(\vec a, \vec b, \epsilon) = \text{SMP}(\vec a, \vec b, f_{\text{SSMP}}, \Psi_{\text{SSMP}})
\end{equation*} 
We write the corresponding decision problem of SSMP as $\text{SSMP}^-(\vec a, \vec b, \epsilon)$ which is an NP-complete problem (this is clear from a reduction to the classical subset sum problem when one side contains the target value, the other side is the input list of integers, and $\epsilon = 0$).

%% file: 4_related_works.tex
\section{Related Works}

Since SSMP is a new CO problem, we seek for similar works while highlighting those that provided inspiration for formulating and solving SSMP instances. 

The Subset Sum Problem (SSP) is closely related to SSMP, which is a canonical CO problem with the goal of finding a subset of numbers that sum up precisely to a target number~\cite{kleinberg2006algorithm}. Search-based algorithms for solving SSP include exhaustive search (e.g., binary tree search) or using heuristics for tree pruning (see, for example,~\cite{horowitz1974computing}). SSP can also be solved in pseudo-polynomial time using dynamic programming (DP) or similar methods~\cite{bringmann2017near,kellerer2003efficient,koiliaris2019faster,pisinger2003dynamic}. The run-time complexity can be further improved by more sophisticated algorithm or data structure design (see, for example, ~\cite{eppstein1997minimum,koiliaris2019faster}). Beyond SSP, as stated in the introduction, there are many other CO problems with similarities to SSMP, such as hypergraph matching~\cite{schrijver2003combinatorial}, maximum clique~\cite{karp2009reducibility}, set packing~\cite{karp2009reducibility}, and optimal multiway partioning~\cite{schreiber2018optimal}.

For some of the techniques used in this paper, we draw inspiration from solvers for the 0-1 Knapsack Problem (KP), where the goal is to maximise the total value in a knapsack with a capacity limit~\cite{martello2000new}. Algorithms for solving 0-1 KP mainly fall into two groups: search, usually branch-and-bound search applying tight upper bounds to guide the search~\cite{dantzig1957discrete,dasgupta1992nonstationary,martello1977upper}; and DP based on the Bellman recursion~\cite{bellman1966dynamic,pisinger1997minimal,pisinger1999linear}. Search and DP can also be combined to solve KP more efficiently~\cite{martello1999dynamic}. 

In addition to search or DP, CO problems like KPs can be interpreted as optimization problems and solved with Mixed-Integer Linear Programming (MILP)~\cite{conforti2014integer}, which is a constrained optimisation problem supported by industrial optimisers like CPLEX~\cite{manual1987ibm}. SSP can also be approximately solved via quadratic optimisation by forming a proxy problem over a convex set~\cite{sahni1974computationally}. Furthermore, frontiers in related areas show a trend in recent works on other CO problems to use reinforcement learning to guide the search~\cite{mazyavkina2021reinforcement}. 

As already noted, reconciliation is an essential function to maintain data integrity and consistency in accounting systems~\cite{ge2005disclosure,vorhies2006account}. It is often solved manually using rules-based systems~\cite{chew2020unsupervised}. In a standard reconciliation problem, 
each individual record in the two sets being reconciled contains a transaction amount and other transactional information, such as a textual description of the transaction. In~\cite{chew2020unsupervised,chew2012automated}, they utilize statistical and machine learning based approaches to match records based on both the transaction amount and the other transactional information. The focus of this work is on matching transactions solely based on the transaction amount, which can be used in conjunction with approaches that match only the transactional information. We demonstrate that matching transaction amounts is a complex, computationally challenging CO problem. An automated solution to reconciliation therefore would benefit from isolating this matching problem from the transactional information matching problem.

%% file: 5_6_solution.tex
\section{An Optimal Algorithm for SSMP via Integer Programming}

\input{_5_1_optimisation}

\section{Faster Sub-optimal Algorithms of SSMPs}
A sub-optimal solution to SSMP can be computed by a greedy algorithm that iteratively creates and solves a series of $\text{SSMP}^-$ tasks with the remaining elements in $\vec a$ and $\vec b$. Alg.~\ref{alg:pseudo-framework} 
presents this algorithm for SSMP assuming access to an algorithm Solve for the decision problem $\text{SSMP}^-$.
In this section, we present two algorithms for Solve based on search and dynamic programming.

 \begin{algorithm}[!htb]
\caption{Sub-optimal Solver for SSMP}
\label{alg:pseudo-framework}
\textbf{Input}:  $\text{SSMP}(\vec a, \vec b, \epsilon)$\\
\textbf{Output}: $s$
\begin{algorithmic}[1] 
\State $s \gets \emptyset$, finish $\gets False$ 
\While{not (finish)}
\State $s^-\gets$ Solve($\text{SSMP}^-(\vec a, \vec b, \epsilon)$) 
\If{ $s^- = \perp$}
\State finish $\gets True$ 
\Else 
\State $s \gets s \cup s^-$ 
\State $\vec a \leftarrow \vec a_{/s^-}$ and $\vec b \leftarrow \vec b_{/s^-}$
\EndIf
\EndWhile
\State \textbf{return} $s$
\end{algorithmic}
\end{algorithm}

\input{_6_1_search}
\input{_6_2_dp}

%% file: _5_1_optimisation.tex
In this section, we provide an algorithm to optimally solve SSMP by modeling it as a mixed integer linear program (MILP). For a given $\text{SSMP}(\vec a, \vec b, \epsilon)$ instance,
where $\vec a = [a_1, a_2,...,a_M]$ and $\vec b = [b_1, b_2, ..., b_N]$ represent two vectors of non-zero real numbers, we introduce three sets:  $\mathcal{M} = \{1,..., M\}$;  $\mathcal{N} = \{1,..., N\}$; and  $\mathcal{Z} = \{1, ..., K\}$ where $K$ is the number of potential matches in a feasible solution. (It suffices to set $K = \min \{M, N\}$ for in each feasible solution, the maximum number of matches is $\min\{M, N\}$, which is attained by feasible solutions that assign each element in the smaller vector to its own match.)
We define the following binary variables, taking value 1 when:
\begin{itemize}
    \item $w^k_i$: element $i \in \mathcal{M}$ is included in match $k \in Z$;
    \item $v^k_j$: element $j \in \mathcal{N}$ is included in match $k \in Z$;
    \item $m_k$: if match $k \in \mathcal{Z}$ is made,
\end{itemize}
yielding a total of $(M \times K) + (N \times K) + K$ variables. Considering $\Psi_{\text{SSMP}}$ and the optimisation task in SSMP introduced in~\eqref{eq:optimal} and~\eqref{eq:SSMP_target_function}, we have modelled the objective function as:
\begin{align}
\max_{m, w, v}\ 
    \sum_{\substack{k \in \mathcal{Z}}}
        m_k
    \ +  \ 
    \sum_{\substack{i \in \mathcal{M},\ k \in \mathcal{Z}}}
        w^k_i
            \ +  \ 
    \sum_{\substack{j \in \mathcal{N},\ k \in \mathcal{Z}}}
        v^k_j
\end{align}

Similar to $\Psi_{\text{SSMP}}$ (Eq.~\eqref{eq:SSMP_target_function}), the first term of the objective function represents the number of generated matches, and the remaining two terms represent the total number of matched elements from each list. In the following we discuss the constraints that are considered for our model.

Based on the problem definition, the sum of each match should be within the threshold $\epsilon \ge 0$. This is enforced by:
\begin{equation}
\begin{split}
    \left|\sum_{i \in \mathcal{M}}a_i w^k_i - \sum_{j \in \mathcal{N}} b_j v^k_j\right| \leq \epsilon ,
    \quad
    \forall k \in \mathcal{Z}
\end{split}
\end{equation}

Eq.~\eqref{eq:mip_one_match} ensures that each element from each list will be included in at most one match.

\begin{equation}
\begin{split}
\label{eq:mip_one_match}
    \sum_{k \in \mathcal{Z}} w^k_i \leq 1,
    \quad
    \sum_{k \in \mathcal{Z}} v^k_j \leq 1,
    \quad
    \forall i \in \mathcal{M},  j \in \mathcal{N}
\end{split}
\end{equation}

Also, we want to enforce that the $k$-th match is generated if and only if there is at least one element for each list:

\begin{equation}
\begin{split}
\label{eq:one_element}
    w^k_i \leq m_k,
    \quad
    v^k_j \leq m_k,
    \quad
    \forall i \in \mathcal{M},  j \in \mathcal{N}, k \in \mathcal{Z}
\end{split}
\end{equation}
We finally want to ensure that if there are no elements in the $k$-th match, then the value of the binary variable associated to that match is 0.

\begin{equation}
\begin{split}
\label{eq:zero}
    \sum_{i \in \mathcal{M}} w^k_i \geq m_k,
    \quad
    \sum_{j \in \mathcal{N}} v^k_j \geq m_k,
    \quad
    \forall k \in \mathcal{Z}
\end{split}
\end{equation}

%% file: _6_1_search.tex
\subsection{Search Solver for $\text{SSMP}^-$}\label{sec:search-solver}
The search solver applies an exhaustive search due to lack of heuristics for SSMPs. Inspired by~\cite{horowitz1974computing}, we use caching to reduce the run-time. The $\text{SSMP}^-$ search solver comprises of two main steps: (a) pre-computing subset sums and storing them in memory and (b) search in memory until a valid \emph{match} is found. The search solver is shown in Alg.~\ref{alg:search_match} and Fig.~\ref{fig:search} is an illustration of the algorithm with an example. 
\input{alg_search}
\begin{figure}[!hbt]
\centering
\includegraphics[width=0.45\textwidth]{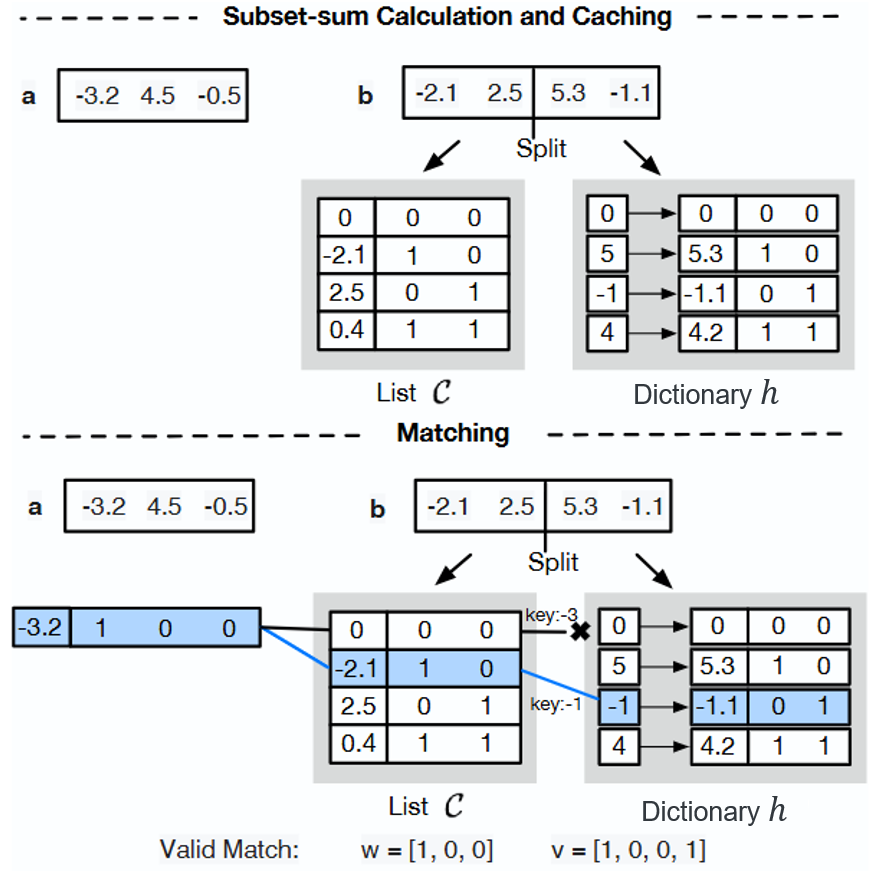}
\caption{\label{fig:search} An example of the search algorithm for solving an $\text{SSMP}^-$ instance. \textbf{Subset-sum Calculation and Caching}: $\vec b$ is split with $r=2$. Subset sums and corresponding inclusion vectors are stored in list $\mathcal C$ and dictionary $h$. \textbf{Matching}: the algorithm generates a subset from $\vec a$ at each iteration and then seeks to match with records saved in memory.}
\end{figure}

\subsubsection{Pre-calculation and Caching}
Pre-calculation and caching can avoid repetitively computing the same combinations during the search procedure. Furthermore, our algorithm parameterizes the space requirement, which is an ideal characteristic of a worst-case exponential space algorithm.

Assume $N \ge M$. We split the longer vector $\vec b$ into two parts $\vec b'$, $\vec b''$ satisfying:
\begin{equation}
    \vec b' \oplus \vec b'' = \vec b,\quad |\vec b'| = r
\end{equation}
where $\oplus$ denotes vector concatenation and $r \in [0, N]$ is the split point. The corresponding inclusion-exclusion vectors for subsets representing $\vec b'$ and $\vec b''$ are $\vec v' \in \{0,1\}^r$ and $\vec v'' \in \{0,1\}^{N-r}$, respectively. We calculate subset sums and perform caching for these two components in slightly different ways:
\begin{itemize}
    \item For all possible $\vec v' \in \{0,1\}^r$, the corresponding subset sum $c$ is calculated and stored in a list $\mathcal{C}$ (note $|\mathcal{C}|$ = $2^r$):
    \begin{equation*}
         \mathcal{C} = \{(c, \vec v'): c = \vec v' \cdot \vec b', \vec v' \in \{0,1\}^r\} 
    \end{equation*}
    \item For all possible $\vec v'' \in \{0,1\}^{N-r}$, we calculate the corresponding subset sum $d$ and store it in the associative array $h(\cdot)$:

    If $\epsilon = 0$,
    \begin{equation*}
        h(d) := \{(d, \vec v'') : d = \vec v'' \cdot \vec b'', \vec v'' \in\{0,1\}^{N-r}\} 
    \end{equation*}
    If $\epsilon > 0$,
    \begin{equation*}
        h(\round{d/\epsilon}) := \{(d, \vec v'') : d = \vec v'' \cdot \vec b'', \vec v'' \in\{0,1\}^{N-r}\} 
    \end{equation*}
    where $\round{x}$ denotes the integer closest to $x$. 
\end{itemize}

\subsubsection{Matching}
The matching procedure in the search-based approach is described in Alg.~\ref{alg:search_match}, which goes through all subsets sums of $\vec a$ and searches for matched values from the subset sums of $\vec b$. With the help of the pre-calculated subset sums stored, the algorithm iterates over each $\vec w \in \{0,1\}^{M}\setminus \{\allzeros\}$ to calculate subset sums of $\vec a$ at run-time while recalling each record $(c, \vec v')$ stored in list $\mathcal{C}$.
The search targets $(d, \vec v'')$ that can form a \emph{match}, which is true if and only if (see Eq.~\eqref{eq:SSMP_matching_criteria}): 
\begin{equation}
\label{eq:target_d}
    d \in [\vec w \cdot \vec a  - c - \epsilon, \vec w \cdot \vec a  - c + \epsilon]
\end{equation}

Suppose $\epsilon > 0$ and there exists a match $|\vec w \cdot \vec a - (\vec v' \cdot \vec b' + \vec v'' \cdot \vec b'')| \le \epsilon$. Rearranging these inequalities, letting $c = \vec v' \cdot \vec b'$ and $d = \vec v'' \cdot \vec b''$,
\begin{equation}\label{eq:round}
  \frac{\vec w \cdot \vec a - c}{\epsilon} - 1 \le \frac{d}{\epsilon} \le \frac{\vec w \cdot \vec a - c}{\epsilon} + 1
\end{equation}
Note that $(d, \vec v'')$ is added to associative array $h$ at $\round{\frac{d}{\epsilon}}$. Let $\hat d_r = \round{(\vec w \cdot \vec a - c) / \epsilon}$. Then Inequalities~(\ref{eq:round}) show that $\hat d_r - 1 \le \round{\frac{d}{\epsilon}} \le \hat d_r + 1$. Therefore, we will find $(d, \vec v '')$ in $h(\hat d_r - 1) \cup h(\hat d_r) \cup h(\hat d_r + 1)$.
Since $c = \vec v '' \cdot \vec b ''$ is one of the sums considered in $\mathcal{C}$ and we added $(d,\vec v'')$ to $h(\round{d/\epsilon})$, we will find the existing match as long as we check $h(\hat d_r - 1), h(\hat d_r), h(\hat d_r + 1)$, which we do in Alg.~\ref{alg:search_match}.

The case of $\epsilon = 0$ is much simpler. Since approximate matches are not allowed, we can simply check if $\vec w \cdot \vec a - c$ exists in $h$ and if so, we return any entry at that key.

\begin{theorem}\label{thm:search}
  If $\mathcal{A},\mathcal{B} \subseteq \mathbb{Z}$, the search algorithm in Alg.~\ref{alg:search_match} solves the $\text{SSMP}^-(\vec a,\vec b,\epsilon)$ problem exactly and runs in $O(2^{N-r} + (1+\epsilon)2^{M+r})$ time and $O(2^r + 2^{N-r})$ space.
\end{theorem}
The proof of the preceding theorem is given in Appendix~\ref{app:skipped}.

\begin{remark}\label{rem:positive-ints}
  We only focus on the integer case in Theorem~\ref{thm:search} as the main application of interest is in financial reconciliation where the data has a constant number of significant digits. Note that the running time depends on the precision as data with higher precision leads to larger numbers of entries per key in $h$ (in the integer case it is only $O(\epsilon)$).
\end{remark}

The search algorithm does not explicitly attempt to make smaller matches as is desired in the objective function. However, one can heuristically do this by considering the smallest sets first in the outer for loop.

\subsubsection{Comparison to Naive Algorithms}\label{sec:naive}
Two natural naive algorithms one could consider for this problem use the extreme values of $r$. When $r = 0$, the algorithm caches everything in the associative array $h$ (could also be implemented as a balanced binary search tree, but this would increase running times). When $r = N$, no associative array $h$ is used, and the algorithm simply iterates over all subset of sums of $\vec b$. In a sense, for $r \in (0, N)$, our algorithm is an ``interpolation" of these two naive algorithms.

We choose a value of $r$ to minimize the running time in Theorem~\ref{thm:search} when $\epsilon = 0$: $r = \frac{N-M}{2}$. This leads to a running time of $O((1+\epsilon)2^{(M+N)/2})$ and $O(2^{(M+N)/2})$ space.

In many cases, our algorithm is \emph{exponentially} better in time and space compared to the naive algorithms. A standard regime of $N$ and $M$ for financial reconciliation is when $N \gg M$. For simplicity, assume $\epsilon = 0$ and $\log N = M$. In this case, for $r = 0$, the running time is $O(2^N)$ and for $r = N$, the running time is $O(2^{N + \log N})$. The running time of our algorithm for $r = \frac{N-M}{2}$ is  $O(2^{(N + \log N)/2})$. This is exponentially better than both running times. The same holds for the space guarantees.

%% file: alg_search.tex
\begin{algorithm}[!htb]
\caption{Search Solver for $\text{SSMP}^-$}
\label{alg:search_match}
\textbf{Input}:  $\text{SSMP}^-(\vec a, \vec b, \epsilon)$\\
\textbf{Output}: $s^-$
\begin{algorithmic}[1] 
\State Generate $\mathcal{C}$, $h$ based on $\vec b$
\For{each $ \vec w \in \{0,1\}^{M}\setminus \{\allzeros\}$, each $(c, \vec v') \in \mathcal{C}$}
\State $\hat{d} \gets \vec w \cdot \vec a - c$
\If{$\epsilon = 0$}
  \State $K \gets \{\hat d\}$
\Else
  \State $\hat d_r \gets \round{(\vec w \cdot \vec a - c) / \epsilon}$
  \State $K \gets \{\hat d_r - 1, \hat d_r, \hat d_r + 1\}$
\EndIf
\For{$(d, \vec v '')$ in $h(k)$ for all $k\in K$} 
  \If{$d \in [\hat d - \epsilon, \hat d + \epsilon]$}
    \State \textbf{return} $s^- \gets \{\langle \vec a_{\vec w}, \vec b_{\vec v' \oplus \vec v''} \rangle$\}
  \EndIf
\EndFor
\EndFor
\State \textbf{return} $s^- \gets \perp$
\end{algorithmic}
\end{algorithm}

%% file: _6_2_dp.tex
\subsection{Dynamic Programming Solver for $\text{SSMP}^-$}\label{sec:dp-solver}
Dynamic programming (DP) for $\text{SSMP}^-$ is a pseudo-polynomial time solver. Like the search-based approach, this DP solver is also a method for finding potential \emph{matches} and then validate them in a post-hoc manner. This algorithm involves three stages: (a) discretisation and element reorganisation to form a proxy integer problem, (b) tabulation, i.e., building DP tables and (c) backtracking until a valid \emph{match} is found. The DP solver is shown in Alg.~\ref{alg:dp_match} and Fig.~\ref{fig:dp} is an illustration of the algorithm with an example. 

\begin{figure}[!hbt]
\centering
\includegraphics[width=0.45\textwidth]{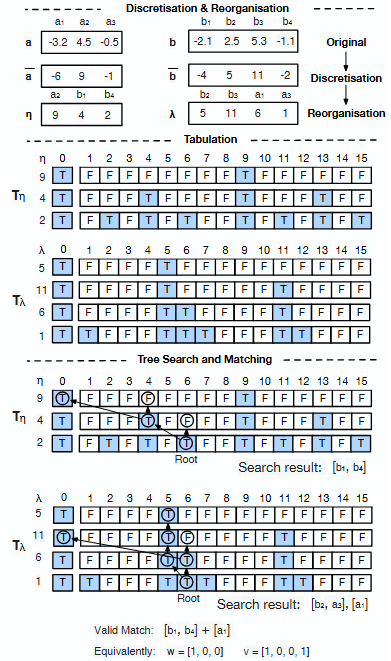}
\caption{\label{fig:dp} Illustration of the DP approach with the same example as shown in Fig.~\ref{fig:search}. \textbf{Discretisation \& Element Reorganisation} In this example converting real numbers to integers via $\rho = 2$. \textbf{Tabulation} Two tables are created from $\vec \eta = [9,4,2], \vec \lambda = [5, 11, 6, 1]$. \textbf{Tree Search and Matching}: We present an example of a tree structure. The first tree with root node $\mat T_{\vec \eta}[3,6]$ returns one subset of $\vec \eta$ which links to $[b_1, b_4]$. The tree starting from $\mat T_{\vec \lambda}[4,6]$ returns 2 subsets of $\vec \lambda$ linking to $[b_2,a_3]$ and $[a_1]$. After validation, only the combination of $[b_1, b_4]$ and $[a_1]$ makes a valid match. Equivalently, $\vec w = [1,0,0], \vec v = [1,0,0,1]$.}
\end{figure}
\input{alg_dp}

\subsubsection{Discretisation and Element Reorganisation}

We use $\bar{\vec a} = [\bar a_1, ...,\bar a_M]$ and $\bar{\vec b} = [\bar b_1, ..., \bar b_N]$ to represent the elements of $\vec a, \vec b$ after applying element-wise discretisation $\bar{\vec a}_m = \round{\rho \vec a_m}, \bar{\vec b}_n = \round{\rho \vec b_n}$ where $\rho > 0$ is a scale (rounding and scaling are not required if the inputs are integers). Then, we reorganise all the elements in $\vec{ \bar a}, \vec{ \bar b}$ into two groups of positive integers:
\begin{align}
    \vec \eta &= [\bar a: \bar a > 0, \bar a \in \vec{ \bar a}] \oplus [-\bar b: \bar b <0 , \bar b \in \vec{ \bar b}] \\
    \vec \lambda &= [-\bar a: \bar a < 0, \bar a \in \vec{ \bar a}] \oplus [\bar b: \bar b > 0 , \bar b \in \vec{ \bar b}]
\end{align}
We use $\vec \eta$ and $\vec \lambda$ to form a new problem $\text{SSMP}^-(\vec \eta, \vec \lambda, \bar \epsilon)$, where each element in $\vec \eta$ and $\vec \lambda$ corresponds to one and only one element in $\vec a$ or $\vec b$. Let $M' = |\vec \eta|$ and $N' = |\vec \lambda|$, we have $ M'+ N' = M+N$. The new matching threshold $\bar \epsilon$ is chosen to be the smallest integer that guarantees every feasible solution of the original $\text{SSMP}^-(\vec a, \vec b, \epsilon)$ can be recovered from a feasible solution of $\text{SSMP}^-(\vec \eta, \vec \lambda, \bar \epsilon)$.

\begin{proposition}\label{prop:dp-mapping}
  Every valid match $\langle \vec a_w, \vec b_v \rangle$ in $\text{SSMP}^-(\vec a,\vec b,\epsilon)$ corresponds to a valid match $\langle \vec \eta_{w'}, \vec \lambda_{v'} \rangle$ in $\text{SSMP}^-(\vec{\eta}, \vec{\lambda}, \bar{\epsilon})$ where $\bar \epsilon = \lceil \rho \epsilon + (M+N)/2 \rceil$. If $\mathcal A, \mathcal B\subseteq \mathbb{Z}$, suffices to set $\bar \epsilon = \epsilon$ where $\rho = 1$.
  
  Given  $\langle \vec \eta_{w'}, \vec \lambda_{v'}\rangle$,  $\langle\vec a_w, \vec b_v\rangle$ can be recovered in $O(M+N)$ time.
\end{proposition}
The proof of the preceding proposition is in Appendix~\ref{app:skipped}.

\input{table1}
\input{table2}

\subsubsection{Tabulation}
We create two tables $\mat T_{\vec \eta}, \mat T_{\vec \lambda}$ for storing the feasibility of achieving a certain subset sum from elements in $\vec \eta, \vec \lambda$ respectively. Taking $\mat T_{\vec \eta}$ as an example, the number of rows is equal to $M'$ while the number of columns is $X+1$ where $X$ is the largest subset value that can be matched across sides, i.e., $X = \min(\sum_{m' = 1}^{M'} \vec \eta_{m'}, \sum_{n' = 1}^{N'} \vec \lambda_{n'})$. Each entry $\mat T_{\vec \eta}[m', i]$ represents whether the subset sum $i$ can be computed with the first $m'$ elements in $\vec \eta$ ($m' \in [1, M'], i \in [0, X]$). The recurrence relation for the table is as follows:
\begin{align}
\label{eq:rule}
    \mat T_{\vec \eta}[m',i] = \begin{cases} 
    False \quad \text{if} \quad m'=1 \quad \text{and} \quad  i \ne \eta_{1} \\ 
    True \quad \text{if}\quad  i = \eta_{m'} \quad \text{or} \quad i= 0 \\
    \mat T_{\vec \eta}[m'-1,i] \lor \mat T_{\vec \eta}[m'-1,i-\eta_{m'}] \quad o.w. 
    \end{cases}
\end{align}
The definitions of $\mat T_{\vec \eta}$, $\mat T_{\vec \lambda}$ match the definition of the recurrence relation used for the standard DP solution to the subset sum problem.

\subsubsection{Tree Search and Matching}\label{sec:dp-tree-search}
With the tables $\mat T_{\vec \eta}, \mat T_{\vec \lambda}$, the matching procedure compares the last row of each table to find the matched subset sums within tolerance $\bar \epsilon$:
\begin{equation}
    \mathcal{H} = \mathcal{H}_0 \cup \mathcal{H}_1 \cup \cdots \cup \mathcal{H}_{\bar \epsilon}
\end{equation}
where
\begin{equation}\label{eq:h_e}
    \mathcal{H}_e = \{(i, j): \mat T_{\vec \eta}[M', i] \land \mat T_{\vec \lambda} [N', j] = True, |i-j| = e \}
\end{equation}
Each $(i,j)$ satisfying this condition would lead to at least one subset in $(\vec \eta, \vec \lambda)$ whose sum is equal to $(i, j)$. The corresponding subset(s) can be discovered by back-tracking in the table via a binary search tree. This is nearly the standard backtracking approach in dynamic programming, with the main difference being we obtain all subsets matching the target. These binary search trees start with the root node located in the last row of the DP table and follow Eq.~\eqref{eq:rule} in the reverse direction; for a tree node $\mat T_{\vec \eta}[m', x]$, if its value is $True$, then $\mat T_{\vec \eta}[m'-1, x]$, $\mat T_{\vec \eta}[m'-1, x-\eta_{m'}]$ are its two child nodes (if it exists). A pruning rule is used to terminate searching that branch when: (a) the value of a node is $False$ or (b) the column index of the node in the table is 0. Fig.~\ref{fig:dp} shows examples of such binary trees. The search path leading to each leaf node whose value is $True$ uniquely links to a subset corresponding to the target subset sum. 

For each pair $(i,j) \in \mathcal{H}$, a binary tree search starting at $\mat T_{\vec \eta}[M',i]$ and $\mat T_{\vec \lambda}[N',j]$ results in two groups of subsets and each pair between the two groups is a possible feasible solution. Let  $\langle \vec \eta_{\vec p},  \vec \lambda_{\vec q} \rangle$ denote an arbitrary combination between the two groups. $\vec p \in \{0,1\}^{M'}$ and $ \vec q \in \{0,1\}^{N'}$ are the vectors for describing subsets of $\vec \eta$ and $\vec \lambda$, respectively, that link to a potential \emph{match} $\langle \vec a_{\vec w}, \vec b_{\vec v} \rangle$ (by reversing the reorganisation process) to the original $\text{SSMP}^-(\vec a, \vec b, \epsilon)$. 

\begin{theorem}\label{thm:dp}
  The DP algorithm solves $\text{SSMP}^-(\vec a, \vec b, \epsilon)$ exactly. When $\mathcal A, \mathcal B \subseteq \mathbb N$, letting $X = \min(\sum_{m = 1}^{M} \vec a_m, \sum_{n = 1}^{N} \vec b_{n})$, the algorithm runs in $O((M+N) \cdot X)$ time and
  $O((M+N)\cdot X)$ space.
\end{theorem}

The proof of the preceding theorem is in Appendix~\ref{app:skipped}.

\begin{remark}
  Generalizing time and space bounds in Theorem~\ref{thm:dp} beyond positive integers results in an exponential running time. The reason for this is that the converse of Proposition~\ref{prop:dp-mapping} is false: a solution in $\text{SSMP}^-(\vec \eta, \vec \lambda, \bar{\epsilon})$ does not necessarily correspond to a solution in $\text{SSMP}^-(\vec a, \vec b, \bar \epsilon)$. When the input is thus not positive integers, the algorithm must enumerate \emph{all possible solutions} in $\text{SSMP}^-(\vec \eta,\vec \lambda,\bar \epsilon)$.
  We again emphasize the positive integer case is important as it directly applies to financial reconciliation (see Remark~\ref{rem:positive-ints}).
\end{remark}

%% file: alg_dp.tex
\begin{algorithm}[!htb]
\caption{Dynamic Programming Solver for $\text{SSMP}^-$}
\label{alg:dp_match}
\textbf{Input}:  $\text{SSMP}^-(\vec a, \vec b, \epsilon)$\\
\textbf{Output}: $s^-$
\begin{algorithmic}[1] 
\State Form $\text{SSMP}^-(\vec \eta, \vec \lambda, \bar \epsilon)$ with discretisation \& reorganisation
\State Build and update tables $\mat T_{\vec \eta}$, $\mat T_{\vec \lambda}$
\For{$0 \le e \le \bar \epsilon$}
\State Find matched integer sums $\mathcal H_e $ from $\mat T_{\vec \eta}$, $\mat T_{\vec \lambda}$
\For{each $ (i, j) \in \mathcal{H}_e$}
\State Collect qualified subsets from tree search:\\
\quad\quad\quad $J = \{\vec \eta_{\vec p}: \vec \eta \cdot \vec p = i\}, \Lambda = \{\vec \lambda_{\vec q}: \vec \eta \cdot \vec q = j\}$
\For{$\vec \eta_{\vec p} \in J$, $\vec \lambda_{\vec q} \in \Lambda$ }
\State Recover $\langle \vec a_{\vec w}, \vec b_{\vec v} \rangle$ from $\langle \vec \eta_{\vec p}, \vec \lambda_{\vec q} \rangle$
\If{$\langle \vec a_{\vec w}, \vec b_{\vec v} \rangle$ is a valid \emph{match}}
\State \textbf{return} $s^- \leftarrow \{\langle \vec a_{\vec w}, \vec b_{\vec v} \rangle$\}
\EndIf
\EndFor
\EndFor
\EndFor
\State \textbf{return} $s^- \gets \perp$
\end{algorithmic}
\end{algorithm}

%% file: table1.tex
\begin{table*}[!tb]
    \centering
    % \small
                \begin{tabular}{c c c| cc | cc | cc } 
                     \toprule 
                     \multicolumn{3}{c}{Task Parameters} & 
                     \multicolumn{2}{c}{Optimisation} & \multicolumn{2}{c}{Search} & \multicolumn{2}{c}{DP} \\
                     \cmidrule(lr){1-3} \cmidrule(lr){4-5}  \cmidrule(lr){6-7} \cmidrule(lr){8-9} 
                    M & N & $\gamma$ & Measure & Run-time (s) &  Measure & Run-time (s)  &  Measure & Run-time (s)   \\
                     \midrule
10 & 10 & $10^4$ & 14.6 (2.0) & 7.5 (4.9) & 7.9 (1.6) & 0.469 (0.201) & 9.2 (3.8) & 0.108 (0.039) \\
10 & 15 & $10^4$ & 18.1\,(2.6)$^*$ & time-out & 16.2 (3.2) & 0.194 (0.084) & 15.5 (2.4) & 0.170 (0.057) \\
10 & 20 & $10^4$ & 22.4\,(2.8)$^*$ & time-out & 21.2 (3.8) & 0.417 (0.609) & 23.4 (3.1) & 0.257 (0.068) \\
10 & 25 & $10^4$ & 26.7\,(3.7)$^*$ & time-out & 24.1 (3.6) & 1.338 (0.110) & 27.0 (2.4) & 0.394 (0.060) \\
10 & 30 & $10^4$ & 31.5\,(2.7)$^*$ & time-out & 31.0 (2.8) & 5.555 (0.332) & 31.3 (4.7) & 0.617 (0.061) \\
\midrule
10 & 20 & $10^2$ & 31.8\,(3.2)$^*$ & time-out & 27.6 (3.8) & 0.220 (0.155) & 27.0 (3.7) & 0.033 (0.004) \\
10 & 20 & $10^3$ & 29.5\,(2.2)$^*$ & time-out & 23.7 (2.8) & 0.176 (0.016) & 24.2 (2.7) & 0.065 (0.011) \\
10 & 20 & $10^4$ & 23.0\,(3.3)$^*$ & time-out & 20.5 (2.9) & 0.348 (0.221) & 22.4 (2.9) & 0.202 (0.037) \\
10 & 20 & $10^5$ & 12.3\,(4.8)$^*$ & time-out & 17.4 (4.4) & 0.724 (0.305) & 16.4 (4.8) & 1.417 (0.199) \\
10 & 20 & $10^6$ & 1.7\,(3.4)$^*$ & time-out & 14.0 (3.1) & 1.754 (1.203) & 14.8 (4.7) & 17.077 (5.315) \\
                     \bottomrule
                \end{tabular}
                \caption{Performance measure and time consumption (`mean (std)') for integer SSMPs. $^*$: run-time allowance reached for (at least) some runs.}
                \label{table:integer_performance}
            \end{table*}

%% file: table2.tex
\begin{table*}[!tb]
    \centering
    % \small
                \begin{tabular}{c c c| c c |c c | c c } 
                     \toprule 
                     \multicolumn{3}{c}{Task Parameters} & 
                     \multicolumn{2}{c}{Optimisation} & \multicolumn{2}{c}{Search} & \multicolumn{2}{c}{DP} \\
                     \cmidrule(lr){1-3} \cmidrule(lr){4-5}  \cmidrule(lr){6-7} \cmidrule(lr){8-9} 
                    M & N & $\epsilon$ & Measure & Run-time (s) &  Measure & Run-time (s)  &  Measure & Run-time (s)   \\
                     \midrule

10 & 10 & 1 & 22.5 (2.2) & 1.5 (0.1) & 18.8 (3.3) & 0.029 (0.003) & 19.1 (2.2) & 0.026 (0.003) \\
20 & 20 & 1 & 49.8\,(1.7)$^*$ & time-out & 39.8 (3.1) & 5.216 (0.360) & 41.9 (2.6) & 0.045 (0.005) \\
30 & 30 & 1 & 74.7\,(2.5)$^*$ & time-out & - (-) & time-out & 64.9 (3.4) & 0.089 (0.014) \\
50 & 50 & 1 & 123.6\,(6.8)$^*$ & time-out & - (-) & time-out & 111.5 (7.0) & 0.255 (0.039) \\
75 & 75 & 1 & 185.6\,(7.0)$^*$ & time-out & - (-) & time-out & 171.4 (6.1) & 1.335 (0.216) \\
100 & 100 & 1 & 244.7\,(12.6)$^*$ & time-out & - (-) & time-out & 233.0 (5.3) & 3.326 (0.364) \\

\midrule
10 & 10 & $10^{-4}$ & 7.0 (5.5) & 29.5 (29.6) & 6.4 (4.8) & 3.959 (2.926) & 5.4 (4.7) & 0.239 (0.323) \\
20 & 20 & $10^{-4}$ & 3.5\,(7.1)$^*$ & time-out & 28.3 (4.3) & 26.257 (14.037) & 27.2 (2.9) & 0.223 (0.161) \\
30 & 30 & $10^{-4}$ & 2.0\,(4.0)$^*$ & time-out & - (-) & time-out & 50.4 (5.2) & 0.703 (0.725) \\
50 & 50 & $10^{-4}$ & 7.7\,(8.6)$^*$ & time-out & - (-) & time-out & 94.4 (5.0) & 2.119 (1.737) \\
75 & 75 & $10^{-4}$ & 4.0\,(8.0)$^*$ & time-out & - (-) & time-out & 151.6 (6.2) & 8.952 (1.182) \\
100 & 100 & $10^{-4}$ & 2.7\,(5.6)$^*$ & time-out & - (-) & time-out & 209.6 (5.3) & 21.545 (1.124) \\
\midrule
10 & 10 & 0 & 0.0 (0.0) & 22.3 (21.7) & 0.0 (0.0) & 7.496 (0.227) & 0.0 (0.0) & 6.906 (1.201) \\
15 & 15 & 0 & 0.0\,(0.0)$^*$ & time-out & - (-) & time-out & 0.0 (0.0) & 34.576 (38.403) \\
20 & 20 & 0 & 0.0\,(0.0)$^*$ & time-out & - (-) & time-out & - (-) & time-out \\
                     \bottomrule
                \end{tabular}
                \caption{Performance measure and time consumption (`mean (std)') for real-value SSMPs. $^*$: run-time allowance reached for some runs.
                }
                \label{table:exp_performance}
            \end{table*}

%% file: 7_experiment_result.tex
\section{Experiments and Results}

This section provides an empirical evaluation of the techniques introduced in this paper, evaluating both integer problems and real-value problems. For each problem configuration, we randomly generated 10 problems and set the run-time allowance to 90 seconds for solving each problem (unless otherwise stated). When run-time allowance is reached, solvers either return the current best solution (optimisation) or failure (search and DP). The machine used to run experiments has an Intel(R) Xeon(R) CPU E3-1585L v5 @ 3.00GHz with 64 GB of RAM. The optimal solver is built using CPLEX 20.1.0.  

\subsection{Integer Problems}
We first test algorithms on integer problems -- $\text{SSMP}(\vec a, \vec b, 0)$ with $\vec a$ and $\vec b$ containing integers uniformly distributed (i.i.d) in $[-\gamma, \gamma]$ where $\gamma \in \mathbb{Z}^+$ is a hyper-parameter. We aim to compare the performance and time consumption under (a) different problem scales, i.e., adjusting number of total elements by varying $N$ and (b) different scales by varying $\gamma$. Hyper-parameters for solving integer problems: 
\begin{itemize}
    \item Search solver: $r = (N - M)/2$.
    \item  DP solver: $\rho = 1$, $\bar \epsilon = \epsilon = 0$ (no discretisation).
\end{itemize}
The choice of $r$ minimizes running time for the search algorithm when $\epsilon = 0$, as shown in Section~\ref{sec:naive}. As this section focuses on integers, there is no need to discretize, so it suffices to set $\rho = 1$.

Table~\ref{table:integer_performance} shows the results of both the performance measure (Eq.~\eqref{eq:optimal}) and solving time for integer-value problems.
When the optimisation solver converges before time-out, it outperforms the other two methods by finding the optimal solution. 
For integer problems, the search solver is more sensitive to the total number of elements in the input -- for double the input size from $M+N = 20$ to $M+N = 40$, the time consumption increased over 11 times from 0.46s to 5.5s. By contrast, the DP solver only varies from 0.1s to 0.62s (6 times), which is expected due to the better running time guarantees. When the problem scale is fixed ($M+N = 30$), the distribution of elements also influences the performance and time. Firstly, problems with smaller value range tend to have more matches (independently of the solver). Secondly, the range of values given by $\gamma$ greatly influences the time complexity of DP as the size of tables are much larger, while it almost has no impact on the search solver. 

\subsection{Real-Value Problems}

Our second set of experiments aims at assessing the scalability and optimality of the different algorithms, as well as the additional complexity brought by real values and a non-zero threshold $\epsilon$. 
We generated six different problem configurations by selecting $M, N \in \{ 10,20,30,50,75,100 \}$ since, in our real-world application of financial reconciliation, the size of a majority of the SSMP problems one faces is less than 100. Both $\vec a$ and $\vec b$ contain real numbers independently and uniformly distributed in $[-100, 100]$. 
Each problem is solved under different matching thresholds: (a) $\epsilon=1$, where we expect many numbers to be matched, (b) $\epsilon = 10^{-4}$, where we expect a mix of matched and unmatched elements, and (c) $\epsilon = 0$, where we expect the number of potential matches to be close to $0$. Hyper-parameters of algorithms to solve real-value problems:
\begin{itemize}
    \item Search solver: $r = (N-M)/2$.
    \item DP solver: $\rho = 1$ for $\epsilon = 1$, $\rho = 10$ for $\epsilon = 10^{-4}$ and $\rho = 10000$ for $\epsilon = 0$. 
\end{itemize}
Even though $\epsilon > 0$ in some cases, it is a small constant  and thus the optimal parameter for $r$ is not significantly different from $\frac{N-M}{2}$ so we keep the same value of $r$.
We found the results not particularly sensitive to the choice of $\rho$, but having $\rho$ increase as $\epsilon$ decreases is a natural guiding principle that was necessary to obtain good results.

Results are shown in Table~\ref{table:exp_performance}.  In general, from the performance perspective, the search and DP solvers show similar performance over their solutions when search does not timeout. The optimisation solver out-performs others with a large $\epsilon$ $( = 1)$ even when it returns sub-optimal solutions for large problems ($M, N \ge 30$). 
However, when a lower $\epsilon$ value is used, SSMP problems have a mixture of matched and unmatched elements ($\epsilon = 10^{-4}$), and the MILP solver struggles to scale up.
Thus, it returns significantly worse results.

With respect to time, the search solver is fast in tiny problems. However, due to its exponential time complexity, it times out before finishing the pre-calculation and caching stage for problems $M,N \ge 30$. The optimal solver also times out on all but the smallest of problems, although it benefits from returning a solution even if it does not finish running. DP shows dominating performance and efficiency in problems with relatively larger size ($M, N \ge 20$) whereas both the optimisation and search solvers fail on these problems, especially when $\epsilon = 10^{-4}$. When $\epsilon = 0$, we aim to check the scalability of these solvers when no matches exist in a problem.  All solvers fail to prove that there is no non-empty solution for some problems within the 90s time limit for $M, N \ge 20$. 

\subsection{Further Exploration}
Based on the experimental results we obtained so far, we found the optimiser algorithm, as expected, is progressing slowly and struggling with harder problems. We designed two additional experiments under the same experimental settings to further explore the solving capability of the optimisation solver by providing (a) additional 90s run-time (180s in total) and (b) a warm start by initialising it with the DP solution. Tables~\ref{table:additional_1} and~\ref{table:additional_2} show selected results under integer and real-value SSMPs. For both integer and real-value problems, doubling the time allowance showed a modest increase in performance (although still not finding the optimal solution in most cases). We also see that compared with the DP results in Table~\ref{table:exp_performance}, the optimisation solver seems not to be improving by using a warm start. 

\input{table3}
\input{table4}

%% file: table3.tex
\begin{table}[!tb]
    \centering
            \small    \begin{tabular}{c c c| c c c } 
                     \toprule 
                     \multicolumn{3}{c}{Parameters} & 
                     \multicolumn{3}{c}{Measure (Optimisation)} \\
                     \cmidrule(lr){1-3} \cmidrule(lr){4-6}
                    M & N & $\gamma$ & original &  + extra 90s  &  + warm start   \\
                     \midrule
10 & 10 & $10^4$ & 14.6 (2.0) & 14.6 (2.0) & 9.2 (3.8) \\
10 & 15 & $10^4$ & 18.1\,(2.6)$^*$ & 19.6\,(1.9)$^*$ & 15.5 (2.4) \\
10 & 20 & $10^4$ & 22.4\,(2.8)$^*$ & 24.2\,(1.7)$^*$ & 23.4 (3.1) \\
10 & 25 & $10^4$ & 26.7\,(3.7)$^*$ & 28.1\,(2.9)$^*$ & 27.0 (2.4) \\
10 & 30 & $10^4$ & 31.5\,(2.7)$^*$ & 31.5\,(2.7)$^*$ & 31.3 (4.7) \\
                     \bottomrule
                \end{tabular}
                \caption{Performance of optimisation solver for integer SSMPs as in Table \ref{table:integer_performance}. We compare  results with 90s time allowance (original), 180s time allowance (+ extra 90s) and 90s allowance with warm start (+warm start).}
                \label{table:additional_1}
            \end{table}

%% file: table4.tex
\begin{table}[!tb]
    \centering
               \small \begin{tabular}{c c c| c c c } 
                     \toprule 
                     \multicolumn{3}{c}{Parameters} & 
                     \multicolumn{3}{c}{Measure (Optimisation)} \\
                     \cmidrule(lr){1-3} \cmidrule(lr){4-6}
                    M & N & $\epsilon$ & original &  + extra 90s  &  + warm start   \\
                     \midrule
10 & 10 & $10^{-4}$ & 7.0 (5.5) & 7.0 (5.5) & 6.0 (4.3) \\
20 & 20 & $10^{-4}$ & 3.5\,(7.1)$^*$ & 3.5\,(7.1)$^*$ & 27.2 (2.9) \\
30 & 30 & $10^{-4}$ & 2.0\,(4.0)$^*$ & 3.0\,(5.2)$^*$ & 50.4 (5.2) \\
50 & 50 & $10^{-4}$ & 7.7\,(8.6)$^*$ & 7.7\,(8.6)$^*$ & 94.4 (5.0) \\
75 & 75 & $10^{-4}$ & 4.0\,(8.0)$^*$ & 4.0\,(8.0)$^*$ & 151.6 (6.2) \\
100 & 100 & $10^{-4}$ & 2.7\,(5.6)$^*$ & 3.5\,(7.1)$^*$ & 209.6 (5.3) \\
                \bottomrule
                \end{tabular}
                \caption{Performance of optimisation solver for real-value SSMPs. We compare results with 90s time allowance (original), 180s time allowance (+extra 90s) and 90s allowance with warm start (+warm start). }
                \label{table:additional_2}
            \end{table}

%% file: 8_discussion.tex
\section{Conclusions and Discussion}
This paper introduced a new CO problem, SSMP, inspired by real-world financial reconciliation tasks. We presented three different algorithms that range from sub-optimal (greedy search and dynamic programming) to optimal (mixed integer linear programming). We also generated a benchmark that can be reused by later works and carried out experiments to compare the different algorithms. 

Potential future work includes developing: (a) more efficient algorithms for SSMP and (b) real-world applications with SSMP algorithms as the core for tasks such as financial reconciliation. Although we focus on SSMP, we believe the SMP framework also fits other real-world scenarios -- for example, a task assignment problem matching workers and tasks given some skills, or student-school matching given preference criteria.

%% file: appendix.tex
\section{Skipped Proofs}\label{app:skipped}
\begin{proof}[Proof of Theorem~\ref{thm:search}]
  To argue correctness, the argument preceding the theorem statement in Section~\ref{sec:search-solver} handles the case when there exists a match for the given $\text{SSMP}^-$ instance. Assuming there does not exist a match, the algorithm will correctly output $\perp$ as it verifies the solution before returning.

  To argue the running time bounds, constructing $\mathcal{C}$ takes $O(2^r)$ time and $h$ takes $O(2^{N-r})$ time. (This is only time to compute sums. We can then use these sums to obtain the final answer in the same time. In practice, one might store all subsets, which would increase the theoretical running time bounds but may be ultimately faster.) The number of iterations of the outer for loop is $O(2^{M+r})$. As for the inner for loop, in the integer case, there are $O(\epsilon)$ unique entries per key in $h$. Therefore, the inner loop only takes $O(\epsilon)$ time, leading to the stated running time.

  The space bound is given by evaluating the sizes of $\mathcal{C}$ and $h$.
\end{proof}

\begin{proof}[Proof of Proposition~\ref{prop:dp-mapping}]
  As noted prior to the proposition statement in Section~\ref{sec:dp-solver}, there is a bijection between elements in $\vec \eta$ and $\vec \lambda$ and elements in $\vec a$ and $\vec b$. This bijection and its inverse are easy to maintain as a simple lookup table with constant time queries per element, leading to the $O(M+N)$ running time.

  Now suppose that $\langle \vec a_w, \vec b_v\rangle$ is a valid match in $\text{SSMP}^-(\vec a,\vec b,\epsilon)$.
  Let $g : \{0,1\}^{M'} \times \{0,1\}^{N'} \to \{0,1\}^M \times \{0,1\}^N$ be the bijection mapping 
  subsets of $\vec \eta$ and $\vec \lambda$ to subsets of $\vec a$ and $\vec b$. Let $g(w', v') = (w,v)$.
  Furthermore, let 
  \begin{align*}
      \vec{\bar a}^+ &:= [\bar a: \bar a > 0, \bar a \in \vec{ \bar a}] \\
      \vec{\bar a}^- &:= [-\bar a: \bar a < 0, \bar a \in \vec{ \bar a}] \\
      \vec{\bar b}^+ &:= [\bar b: \bar b > 0, \bar b \in \vec{ \bar b}] \\
      \vec{\bar b}^- &:= [-\bar b: \bar b < 0, \bar b \in \vec{ \bar b}]
  \end{align*}
  With these definitions, $\vec \eta = \vec{\bar a}^+ \oplus \vec{\bar b}^-$ and $\vec \lambda = \vec{\bar a}^- \oplus \vec{\bar b}^+$. Also let $\vec w' = \vec {w'_1} \oplus \vec{w_2'}$ and $\vec v' = \vec {v'_1} \oplus \vec{v'_2}$. Then
  \begin{align*}
      &|\vec w' \cdot \vec \eta - \vec v'\cdot \vec \lambda|\\
      =&|\vec {\bar a}^+ \cdot \vec {w'_1} + \vec{\bar b}^-\cdot \vec{w'_2} - (\vec{\bar a}^- \cdot \vec {v'_1} + \vec{\bar b}^+ \cdot \vec{v'_2})|\\
      =&|\vec {\bar a}^+  \cdot \vec {w'_1} - \vec{\bar a}^- \cdot \vec {v'_1}  - (- \vec{\bar b}^-\cdot \vec{w'_2} + \vec{\bar b}^+ \cdot \vec{v'_2})|\\
      =&|\vec {\bar a} \cdot \vec w - \vec{\bar b} \cdot \vec v|\\
      \le&|\vec {\bar a} \cdot \vec w - \rho \vec a \cdot \vec w| + |\vec {\bar b} \cdot \vec v - \rho \vec b \cdot \vec v| + |\rho \vec a \cdot \vec w - \rho \vec b \cdot \vec v| \\
      \le& \frac{M}{2} + \frac{N}{2} + \rho \epsilon\\
      \le&\bar \epsilon,
  \end{align*}
  where the third equality follows as $g(w',v') = (w,v)$ and the second inequality follows as $\vec {\bar a}$ and $\vec {\bar b}$ are constructed from $\vec a$ and $\vec b$ via rounding and $\langle \vec a_w, \vec b_v\rangle$ is a valid match in $\text{SSMP}^-(\vec a, \vec b, \epsilon)$. 

  Note that if $\rho = 1$ and $\mathcal A, \mathcal B \subseteq \mathbb Z$, then 
  $\vec {\bar a} = \vec a$ and $\vec {\bar b} = \vec b$. Therefore, the above argument can easily show that
  $|\vec w' \cdot \eta - \vec v' \cdot \lambda| \le \epsilon$ in the integer case.
  This concludes the proof.
\end{proof}

\begin{proof}[Proof of Theorem~\ref{thm:dp}]
  Regarding correctness, we first claim that the DP algorithm iterates over all valid matches in $\text{SSMP}^-(\vec \eta, \vec \lambda, \bar \epsilon)$. Let $\langle \vec \eta_{w'}, \vec \lambda_{v'}\rangle$ be a valid match in $\text{SSMP}^-(\vec \eta, \vec \lambda, \bar \epsilon)$. Letting  $|\vec w' \cdot \vec \eta - \vec v' \cdot \vec \lambda| = e$, we then have $e \le \bar \epsilon$. Therefore, we will discover this valid match in the DP algorithm when considering $\mathcal H_e$ as $\mat T_{\vec \eta}[M', \vec w' \cdot \vec \eta] = True$ and $\mat T_{\vec \lambda} [N', \vec v' \cdot \vec \lambda ] = True$.

  With this claim, by Proposition~\ref{prop:dp-mapping}, it is clear that we will ultimately either encounter a valid match in $\text{SSMP}^-(\vec a, \vec b, \epsilon)$ if it exists, or we will correctly find that no such valid match exists. Therefore, the DP algorithm solves $\text{SSMP}^-(\vec a, \vec b, \epsilon)$ exactly.

  For the positive integer case, it is important to observe that $\vec \eta = \vec a$ and $\vec \lambda  = \vec b$. This will change two important aspects of the algorithm: 
  (1) we can set $\bar \epsilon = \epsilon$ and
  (2) we do not have to iterate over all solutions in $\mat T_{\vec \eta}$ and $\mat T_{\vec \lambda}$ as solutions in $\text{SSMP}^-(\vec \eta, \vec \lambda, \epsilon)$ are solutions in $\text{SSMP}^-(\vec a, \vec b, \epsilon)$.
  
  To analyze the running time for the positive integer case, it takes $O((M+N)\cdot X)$ time to construct $\mat T_{\vec \eta}$ and $\mat T_{\vec \lambda}$. In the positive integer case, as noted above, any solution in $\text{SSMP}^-(\vec \eta, \vec \lambda, \epsilon)$ is a solution in $\text{SSMP}^-(\vec a, \vec b, \epsilon)$. Therefore, any pair $(i,j) \in \mathcal{H}$ corresponds to a solution in $\text{SSMP}^-(\vec a, \vec b, \epsilon)$. It takes $O(M+N)$ time to find one pair in $\mathcal{H}$. This can easily be done by simple accounting when iterating over both of the last rows in $\mat T_{\vec \eta}$ and $\mat T_{\vec \lambda}$ (e.g., iterate over one array while keeping track of whether the corresponding indices in the other have at least one $True$). 
  Finally, for any solution pair $(i,j) \in \mathcal{H}$, it takes $O((M+N)\cdot X)$ time to perform the binary search algorithm described in Section~\ref{sec:dp-tree-search} as it takes $O((M+N)\cdot X)$ time to find the candidate sums in each table. 
  This proves the bound on the running time.

  The space bound follows from the sizes of the tables $\mat T_{\vec \eta}$ and $\mat T_{\vec \lambda}$.
\end{proof}

%% file: main.bbl
\begin{thebibliography}{27}
\providecommand{\natexlab}[1]{#1}
\providecommand{\url}[1]{\texttt{#1}}
\expandafter\ifx\csname urlstyle\endcsname\relax
  \providecommand{\doi}[1]{doi: #1}\else
  \providecommand{\doi}{doi: \begingroup \urlstyle{rm}\Url}\fi

\bibitem[Bellman(1966)]{bellman1966dynamic}
R.~Bellman.
\newblock Dynamic programming.
\newblock \emph{Science}, 153\penalty0 (3731):\penalty0 34--37, 1966.

\bibitem[Bringmann(2017)]{bringmann2017near}
K.~Bringmann.
\newblock A near-linear pseudopolynomial time algorithm for subset sum.
\newblock In \emph{Proceedings of the Twenty-Eighth Annual ACM-SIAM Symposium on Discrete Algorithms}, pages 1073--1084. SIAM, 2017.

\bibitem[Chew(2020)]{chew2020unsupervised}
P.~Chew.
\newblock Unsupervised-learning financial reconciliation: a robust, accurate approach inspired by machine translation.
\newblock In \emph{Proceedings of the First ACM International Conference on AI in Finance}, pages 1--12, 2020.

\bibitem[Chew and Robinson(2012)]{chew2012automated}
P.~A. Chew and D.~G. Robinson.
\newblock Automated account reconciliation using probabilistic and statistical techniques.
\newblock \emph{International Journal of Accounting \& Information Management}, 20\penalty0 (4):\penalty0 322--334, 2012.

\bibitem[Conforti et~al.(2014)Conforti, Cornu{\'e}jols, Zambelli, et~al.]{conforti2014integer}
M.~Conforti, G.~Cornu{\'e}jols, G.~Zambelli, et~al.
\newblock \emph{Integer programming}, volume 271.
\newblock Springer, 2014.

\bibitem[Dantzig(1957)]{dantzig1957discrete}
G.~B. Dantzig.
\newblock Discrete-variable extremum problems.
\newblock \emph{Operations research}, 5\penalty0 (2):\penalty0 266--288, 1957.

\bibitem[Dasgupta and McGregor(1992)]{dasgupta1992nonstationary}
D.~Dasgupta and D.~R. McGregor.
\newblock Nonstationary function optimization using the structured genetic algorithm.
\newblock In \emph{PPSN}, volume~2, pages 145--154. Citeseer, 1992.

\bibitem[Eppstein(1997)]{eppstein1997minimum}
D.~Eppstein.
\newblock Minimum range balanced cuts via dynamic subset sums.
\newblock \emph{Journal of Algorithms}, 23\penalty0 (2):\penalty0 375--385, 1997.

\bibitem[Ge and McVay(2005)]{ge2005disclosure}
W.~Ge and S.~McVay.
\newblock The disclosure of material weaknesses in internal control after the sarbanes-oxley act.
\newblock \emph{Accounting Horizons}, 19\penalty0 (3):\penalty0 137--158, 2005.

\bibitem[Horowitz and Sahni(1974)]{horowitz1974computing}
E.~Horowitz and S.~Sahni.
\newblock Computing partitions with applications to the knapsack problem.
\newblock \emph{Journal of the ACM (JACM)}, 21\penalty0 (2):\penalty0 277--292, 1974.

\bibitem[Karp(2009)]{karp2009reducibility}
R.~M. Karp.
\newblock Reducibility among combinatorial problems.
\newblock In \emph{50 Years of Integer Programming 1958-2008: from the Early Years to the State-of-the-Art}, pages 219--241. Springer, 2009.

\bibitem[Kellerer et~al.(2003)Kellerer, Mansini, Pferschy, and Speranza]{kellerer2003efficient}
H.~Kellerer, R.~Mansini, U.~Pferschy, and M.~G. Speranza.
\newblock An efficient fully polynomial approximation scheme for the subset-sum problem.
\newblock \emph{Journal of Computer and System Sciences}, 66\penalty0 (2):\penalty0 349--370, 2003.

\bibitem[Kleinberg and Tardos(2006)]{kleinberg2006algorithm}
J.~Kleinberg and E.~Tardos.
\newblock \emph{Algorithm design}.
\newblock Pearson Education India, 2006.

\bibitem[Koiliaris and Xu(2019)]{koiliaris2019faster}
K.~Koiliaris and C.~Xu.
\newblock Faster pseudopolynomial time algorithms for subset sum.
\newblock \emph{ACM Transactions on Algorithms (TALG)}, 15\penalty0 (3):\penalty0 1--20, 2019.

\bibitem[Manual(1987)]{manual1987ibm}
C.~U. Manual.
\newblock Ibm ilog cplex optimization studio.
\newblock \emph{Version}, 12:\penalty0 1987--2018, 1987.

\bibitem[Martello and Toth(1977)]{martello1977upper}
S.~Martello and P.~Toth.
\newblock An upper bound for the zero-one knapsack problem and a branch and bound algorithm.
\newblock \emph{European Journal of Operational Research}, 1\penalty0 (3):\penalty0 169--175, 1977.

\bibitem[Martello et~al.(1999)Martello, Pisinger, and Toth]{martello1999dynamic}
S.~Martello, D.~Pisinger, and P.~Toth.
\newblock Dynamic programming and strong bounds for the 0-1 knapsack problem.
\newblock \emph{Management science}, 45\penalty0 (3):\penalty0 414--424, 1999.

\bibitem[Martello et~al.(2000)Martello, Pisinger, and Toth]{martello2000new}
S.~Martello, D.~Pisinger, and P.~Toth.
\newblock New trends in exact algorithms for the 0--1 knapsack problem.
\newblock \emph{European Journal of Operational Research}, 123\penalty0 (2):\penalty0 325--332, 2000.

\bibitem[Mazyavkina et~al.(2021)Mazyavkina, Sviridov, Ivanov, and Burnaev]{mazyavkina2021reinforcement}
N.~Mazyavkina, S.~Sviridov, S.~Ivanov, and E.~Burnaev.
\newblock Reinforcement learning for combinatorial optimization: A survey.
\newblock \emph{Computers \& Operations Research}, 134:\penalty0 105400, 2021.

\bibitem[{\"O}ncan(2007)]{oncan2007survey}
T.~{\"O}ncan.
\newblock A survey of the generalized assignment problem and its applications.
\newblock \emph{INFOR: Information Systems and Operational Research}, 45\penalty0 (3):\penalty0 123--141, 2007.

\bibitem[Pisinger(1997)]{pisinger1997minimal}
D.~Pisinger.
\newblock A minimal algorithm for the 0-1 knapsack problem.
\newblock \emph{Operations Research}, 45\penalty0 (5):\penalty0 758--767, 1997.

\bibitem[Pisinger(1999)]{pisinger1999linear}
D.~Pisinger.
\newblock Linear time algorithms for knapsack problems with bounded weights.
\newblock \emph{Journal of Algorithms}, 33\penalty0 (1):\penalty0 1--14, 1999.

\bibitem[Pisinger(2003)]{pisinger2003dynamic}
D.~Pisinger.
\newblock Dynamic programming on the word ram.
\newblock \emph{Algorithmica}, 35\penalty0 (2):\penalty0 128--145, 2003.

\bibitem[Sahni(1974)]{sahni1974computationally}
S.~Sahni.
\newblock Computationally related problems.
\newblock \emph{SIAM Journal on computing}, 3\penalty0 (4):\penalty0 262--279, 1974.

\bibitem[Schreiber et~al.(2018)Schreiber, Korf, and Moffitt]{schreiber2018optimal}
E.~L. Schreiber, R.~E. Korf, and M.~D. Moffitt.
\newblock Optimal multi-way number partitioning.
\newblock \emph{Journal of the ACM (JACM)}, 65\penalty0 (4):\penalty0 1--61, 2018.

\bibitem[Schrijver et~al.(2003)]{schrijver2003combinatorial}
A.~Schrijver et~al.
\newblock \emph{Combinatorial optimization: polyhedra and efficiency}, volume~24.
\newblock Springer, 2003.

\bibitem[Vorhies(2006)]{vorhies2006account}
J.~B. Vorhies.
\newblock Account reconciliation: An underappreciated control.
\newblock \emph{Journal of Accountancy}, 202\penalty0 (3):\penalty0 59, 2006.

\end{thebibliography}
